\documentclass[11pt]{article}

\usepackage[margin=1in]{geometry}
\usepackage{setspace}
\usepackage{amsmath,amsfonts,amssymb}
\usepackage{amsthm}
\usepackage{graphicx}
\usepackage{booktabs}
\usepackage{algorithm}
\usepackage{algorithmic}
\usepackage{array}
\usepackage{textcomp}
\usepackage{mathtools}
\usepackage{xcolor}
\usepackage{nicefrac}
\usepackage{microtype}
\usepackage{hyperref}

\usepackage[round,authoryear]{natbib}
\theoremstyle{plain}
\newtheorem{theorem}{Theorem}[section]
\newtheorem{proposition}[theorem]{Proposition}

\theoremstyle{definition}

\title{\Large \textbf{CADENT: Gated Hybrid Distillation for Sample-Efficient Transfer in Reinforcement Learning}}

\author{
\textbf{Mahyar Alinejad}, University of Central Florida, \texttt{mahyar.alinejad@ucf.edu} \and
\textbf{Yue Wang}, University of Central Florida, \texttt{yue.wang@ucf.edu} \and
\textbf{George Atia}, University of Central Florida, \texttt{george.atia@ucf.edu}
}

\date{}

\begin{document}
\sloppy
\maketitle

\begin{abstract}
Transfer learning promises to reduce the high sample complexity of deep reinforcement learning (RL), yet existing methods struggle with domain shift between source and target environments. Policy distillation provides powerful tactical guidance but fails to transfer long-term strategic knowledge, while automaton-based methods capture task structure but lack fine-grained action guidance. This paper introduces Context-Aware Distillation with Experience-gated Transfer (CADENT), a framework that unifies strategic automaton-based knowledge with tactical policy-level knowledge into a coherent guidance signal. CADENT's key innovation is an experience-gated trust mechanism that dynamically weighs teacher guidance against the student's own experience at the state-action level, enabling graceful adaptation to target domain specifics. Across challenging environments, from sparse-reward grid worlds to continuous control tasks, CADENT achieves 40-60\% better sample efficiency than baselines while maintaining superior asymptotic performance, establishing a robust approach for adaptive knowledge transfer in RL.
\end{abstract}

\section{INTRODUCTION}

Deep Reinforcement Learning (RL) has achieved remarkable success in solving complex sequential decision-making problems, from mastering strategic games~\citep{silver2016mastering,silver2017mastering} to controlling robotic systems~\citep{levine2016end,andrychowicz2020learning}. However, a fundamental challenge persists: the formidable sample complexity required to learn effective policies from scratch. In many real-world scenarios, acquiring millions of environment interactions is impractical or prohibitively expensive~\citep{dulac2019challenges}. Transfer learning has emerged as a powerful paradigm to mitigate this challenge by enabling an agent to leverage knowledge acquired in a source task to accelerate learning in a new, related target task~\citep{taylor2009transfer,zhu2023transfer}.

\subsection{Knowledge Transfer in Reinforcement Learning}

The central question in transfer for RL is what knowledge to transfer and how to transfer it effectively. Early approaches focused on transferring low-level knowledge, such as value functions~\citep{taylor2007transfer}, feature representations~\citep{barreto2017successor}, or learned models~\citep{van2019use}. While effective under conditions of high task similarity, these methods are often brittle and susceptible to negative transfer when faced with significant shifts in state-action spaces or environment dynamics~\citep{lazaric2008transfer,taylor2009transfer}.

\paragraph{Policy-Level Transfer.}
A more robust line of work focuses on transferring behavioral knowledge at the policy level. \textbf{Policy Distillation}~\citep{rusu2015policy}, inspired by seminal work in model compression~\citep{hinton2015distilling}, trains a student agent to mimic the soft action probabilities of a pre-trained teacher policy. This approach has been extended to multi-task learning~\citep{parisotto2015actor} and continual learning settings~\citep{teh2017distral}. Variants include using demonstrations~\citep{hester2018deep}, combining imitation with reinforcement~\citep{rajeswaran2017learning}, and transferring through option discovery~\citep{fox2017multi}. This form of transfer provides powerful, state-specific \textit{tactical} guidance, answering the question of ``how should I act now?'' However, it is fundamentally myopic; it does not explicitly transfer the long-term \textit{strategic} knowledge of how to sequence sub-tasks to achieve a complex goal.

\paragraph{Structure-Based Transfer.}
Orthogonal to policy-level transfer, another research direction leverages formal methods to transfer high-level task structure. By representing tasks as finite automata~\citep{icarte2018using,icarte2022reward} or using Linear Temporal Logic (LTL) specifications~\citep{littman2017environment,camacho2019ltl}, agents can be endowed with an explicit understanding of the task's sequential and logical constraints. Hierarchical RL approaches similarly decompose tasks into sub-goals~\citep{sutton1999between,bacon2017option,nachum2018data}. Recent work on \textbf{Automaton Distillation}~\citep{singireddy2023automaton, alinejad2025bidirectional} abstracts a teacher's successful trajectories into a compact automaton and uses progress within this automaton as an intrinsic reward signal for the student. These approaches effectively transfer a strategic blueprint, answering the question of ``what should I be trying to achieve?'' Yet, they lack the fine-grained, tactical advice on how to best execute the actions required to advance that strategy.

\paragraph{Adaptive Transfer Mechanisms.}
Recent work has begun exploring adaptive transfer mechanisms. Curriculum learning methods gradually increase task difficulty~\citep{narvekar2020curriculum,florensa2017reverse}, while meta-learning approaches learn to adapt quickly to new tasks~\citep{finn2017model,rakelly2019efficient}. Successor features enable generalization across reward functions~\citep{barreto2017successor,barreto2020fast}. However, these methods do not explicitly address the question of when to trust teacher knowledge versus one's own experience in the presence of domain shift.

\subsection{The Gap and Our Contribution}

This reveals a critical gap in the literature: existing methods excel at transferring either tactical policies or strategic task structures, but not both. Furthermore, they typically employ static transfer mechanisms, where the teacher's knowledge is treated as infallible, failing to address the crucial question of \textit{when} a student should deviate from the teacher's advice to adapt to the specific nuances of the target environment. This can lead to suboptimal policies or even negative transfer when the source and target domains differ significantly~\citep{taylor2009transfer,lazaric2008transfer}.

In this paper, we bridge this gap by introducing \textbf{CADENT: Context-Aware Distillation with Experience-gated Transfer}. CADENT is a novel framework that makes the following contributions:
\begin{enumerate}
    \item This paper proposes a \textbf{hybrid distillation framework} that, for the first time, unifies long-term, automaton-based strategic guidance with short-term, policy-based tactical advice into a single, coherent learning signal.
    \item A novel \textbf{experience-gated trust mechanism} is introduced, which allows the student agent to dynamically arbitrate between the teacher's static knowledge and its own evolving experience at the state-action level, enabling robust adaptation and mitigating negative transfer.
    \item Through extensive experiments on a suite of challenging environments—from complex grid worlds requiring deep exploration (\texttt{DungeonQuest}, \texttt{BlindCraftsman}) to control tasks with resource constraints (\texttt{MountainCar}, \texttt{WarehouseRobotics})—we demonstrate that CADENT achieves 40-60\% better sample efficiency while maintaining superior asymptotic performance compared to state-of-the-art transfer learning baselines.
\end{enumerate}

This work establishes a new paradigm for adaptive knowledge transfer in RL, moving beyond static imitation towards a dynamic partnership between teacher and student that gracefully handles domain shift.

\section{PRELIMINARIES}

In this section, we formalize the key concepts that form the foundation of our work: Reinforcement Learning via Markov Decision Processes, and the specific transfer learning paradigms of Policy Distillation and Automaton-based RL.

\textbf{Markov Decision Processes.}
An agent-environment interaction is modeled as a Markov Decision Process (MDP), defined by the tuple $\mathcal{M} = (\mathcal{S}, \mathcal{A}, P, R, \gamma)$. Here, $\mathcal{S}$ is the set of states, $\mathcal{A}$ is the set of actions, $P: \mathcal{S} \times \mathcal{A} \times \mathcal{S} \to [0, 1]$ is the state transition probability function, $R: \mathcal{S} \times \mathcal{A} \to \mathbb{R}$ is the reward function, and $\gamma \in [0, 1)$ is the discount factor.

An agent's behavior is described by a policy $\pi: \mathcal{S} \to \Delta(\mathcal{A})$, where $\pi(a \mid s)$ is the probability of taking action $a$ in state $s$. The agent aims to find an optimal policy $\pi^*$ that maximizes the expected discounted return $G_0 = \sum_{k=0}^{\infty} \gamma^k R_{k+1}$.

The value of a policy is quantified by the state-value and action-value functions $V^\pi(s) = \mathbb{E}_\pi[G_0 \mid S_0=s]$ and $Q^\pi(s,a) = \mathbb{E}_\pi[G_0 \mid S_0=s, A_0=a]$.

Let $r(s,a) = \mathbb{E}[R_1 \mid S_0=s, A_0=a]$ denote the expected immediate reward and $P(\cdot \mid s,a)$ the transition kernel. The optimal action-value function $Q^*(s,a) = \max_\pi Q^\pi(s,a)$ satisfies the Bellman optimality equation
\begin{equation}
Q^*(s,a) = r(s,a) + \gamma \mathbb{E}_{s' \sim P(\cdot \mid s,a)} \left[ \max_{a' \in \mathcal{A}} Q^*(s', a') \right]
\end{equation}

Many RL algorithms, including Q-learning~\citep{watkins1992q}, iteratively solve this equation. In tabular form, given a transition $(s,a,r,s')$, the update is
\begin{equation}
Q(s,a) \leftarrow Q(s,a) + \alpha \left[ r + \gamma \max_{a'} Q(s', a') - Q(s, a) \right]
\end{equation}
where $\alpha$ is the learning rate.

\textbf{Transfer Learning in Reinforcement Learning.}
A standard transfer learning setting considers a source MDP, $\mathcal{M}_{src}$, and a target MDP, $\mathcal{M}_{tgt}$. While they share the same action space $\mathcal{A}$, their state spaces, transition dynamics, and reward functions may differ. The objective is to leverage knowledge extracted from a teacher agent trained on $\mathcal{M}_{src}$ to improve the learning performance of a student agent in $\mathcal{M}_{tgt}$.

\paragraph{Policy Distillation}
Policy Distillation~\citep{rusu2015policy} transfers knowledge by training a student policy $\pi_{student}$ to match the softened action-probability distribution of a teacher policy $\pi_{teacher}$. The teacher's distribution is generated by applying a softmax function with a temperature parameter $\tau > 1$ to its learned Q-values, $Q_{teacher}$:
\begin{equation}
\pi_{teacher}(a|s) = \frac{\exp(Q_{teacher}(s,a)/\tau)}{\sum_{a' \in \mathcal{A}} \exp(Q_{teacher}(s,a')/\tau)}
\end{equation}
Using $\tau > 1$ softens the distribution, providing richer information about the teacher's relative preferences for actions. The student is then trained using a loss function that encourages its own policy, $\pi_{student}$, to match this distribution, typically by minimizing the Kullback-Leibler (KL) divergence, $\mathcal{L}_{PD} = D_{KL}(\pi_{teacher} || \pi_{student})$. This provides fine-grained, \textit{tactical} guidance.

\paragraph{Automaton-based Task Representation}
For tasks with complex, sequential goal structures, a Deterministic Finite Automaton (DFA) can be used to represent the high-level task specification. A DFA is a tuple $\mathcal{D} = (\mathcal{Q}, \Sigma, \delta, q_0, F)$, where $\mathcal{Q}$ is a finite set of automaton states, $\Sigma$ is an alphabet of symbols corresponding to environment observations, $\delta: \mathcal{Q} \times \Sigma \to \mathcal{Q}$ is the transition function, $q_0$ is the start state, and $F \subseteq \mathcal{Q}$ is the set of accepting (final) states.

To leverage this structure, the agent solves a product MDP, $\mathcal{M} \times \mathcal{D}$, with an augmented state space $\mathcal{S}' = \mathcal{S} \times \mathcal{Q}$. An agent in state $(s, q)$ transitions to $(s', q')$ upon taking action $a$ if the environment transitions to $s'$ producing observation $l \in \Sigma$, and $\delta(q, l) = q'$. This framework allows for the design of intrinsic rewards based on progress within the automaton, providing high-level, \textit{strategic} guidance.

\section{PROBLEM FORMULATION}

This paper addresses the challenge of sample-efficient RL in a target task by transferring knowledge from a related source task. Let the source task be represented by an MDP $\mathcal{M}_{src} = (\mathcal{S}_{src}, \mathcal{A}, P_{src}, R_{src}, \gamma)$ and the target task by $\mathcal{M}_{tgt} = (\mathcal{S}_{tgt}, \mathcal{A}, P_{tgt}, R_{tgt}, \gamma)$. The tasks share a common action space $\mathcal{A}$ and discount factor $\gamma$, but may differ in their state spaces, transition dynamics, and reward functions.

The approach assumes access to a teacher policy, $\pi_{teacher}$, that has been pre-trained to near-optimality in the source task $\mathcal{M}_{src}$. The teacher's knowledge is encapsulated in its action-value function, $Q_{teacher}(s_{src}, a)$, and a high-level task automaton, $\mathcal{D}$, which is either provided a priori or inferred from successful teacher trajectories.

The objective is to train a student agent with policy $\pi_{student}$ in the target task $\mathcal{M}_{tgt}$ to converge to the optimal policy $\pi_{tgt}^*$ as quickly as possible. The core challenge is to design a transfer mechanism that leverages the teacher's strategic and tactical knowledge to accelerate learning while remaining robust to the inevitable domain shift between $\mathcal{M}_{src}$ and $\mathcal{M}_{tgt}$. The student must learn not only to imitate the teacher but also to identify when the teacher's knowledge is suboptimal in the new context and adapt accordingly.

\section{METHODOLOGY: THE CADENT FRAMEWORK}

To address this challenge, we propose Context-Aware Distillation with Experience-gated Transfer (CADENT), a novel transfer algorithm that integrates multi-level teacher guidance with a dynamic arbitration mechanism that governs the student's reliance on this guidance. The framework is built on two core principles: (1) unifying the teacher's strategic and tactical knowledge into a single, coherent signal, and (2) gating the influence of this signal based on the student's own accumulated, state-action specific experience in the target environment.

\subsection{Hybrid Distillation: Unifying Strategic and Tactical Guidance}

Traditional methods transfer either high-level strategy or low-level tactics. CADENT fuses both into a stable, multi-faceted guidance signal. The guidance is decoupled into two components: an intrinsic reward for strategic progress and a policy gradient for tactical alignment.

\paragraph{Strategic Guidance as Intrinsic Reward.}
The teacher's strategic knowledge, embodied in the automaton $\mathcal{D}$, provides a powerful signal for long-term planning. To extract this knowledge, the teacher's Q-values are analyzed to identify which automaton transitions are most valuable. Specifically, for each automaton transition $(q, q')$, the distilled transition value $Q_{AD}(q, q')$ is computed by averaging the Q-values of all state-action pairs that trigger this transition:
\begin{equation}
Q_{AD}(q, q') = \frac{1}{|\mathcal{T}_{q \to q'}|} \sum_{(s,a) \in \mathcal{T}_{q \to q'}} Q_{teacher}(s, a)
\end{equation}
where $\mathcal{T}_{q \to q'} = \{(s,a): s = (s_{env}, q) \text{ and action } a \text{ leads to automaton state } q'\}$ is the set of state-action pairs that cause the transition from $q$ to $q'$.

This distilled knowledge is then formulated as an intrinsic reward, $r_{AD}$, that the student receives upon making a meaningful transition in the task automaton. For a student transition from product state $(s, q)$ to $(s', q')$, where $q, q' \in \mathcal{Q}$, the strategic reward is:
\begin{equation}
r_{AD} = 
\begin{cases} 
    \lambda_{AD} \cdot Q_{AD}(q, q') & \text{if } q \neq q' \\
    0 & \text{otherwise}
\end{cases}
\end{equation}
where $\lambda_{AD}$ is a scaling hyperparameter. This reward incentivizes the student to follow the teacher's high-level task completion strategy by providing positive reinforcement for progressing through the automaton in the same manner the teacher learned was valuable.

\paragraph{Tactical Guidance as a Policy Prior.}
The teacher's tactical knowledge is captured by its distilled policy, $\pi_{teacher}(a | q)$, where $q$ is the current automaton state. This policy acts as a powerful prior for action selection, providing context-aware guidance based on the current task phase. This guidance is integrated directly into the learning update using a policy gradient-style correction term, $g_{PD}$. This term nudges the student's policy, $\pi_{student}$, towards the teacher's:
\begin{equation}
g_{PD}(s, a) = \lambda_{PD} \cdot (\pi_{teacher}(a | q) - \pi_{student}(a | s))
\end{equation}
where $q$ is the automaton state component of the augmented state $s = (s_{env}, q)$, and $\lambda_{PD}$ is a hyperparameter controlling the strength of the tactical guidance. This term is a stable, bounded signal that encourages mimicry without destructively overriding the student's value estimates.

\subsection{Experience-Gated Trust Mechanism}

The cornerstone of CADENT is its ability to adapt. A state-action trust metric, $\omega(s,a) \in [0, 1]$, is introduced that quantifies the student's confidence in its own learned value, $Q_{student}(s,a)$. The trust is inversely related to the volatility of the value estimate for that specific state-action pair.

A volatility tracker, $V_t(s,a)$, maintains a running estimate of the magnitude of the student's TD-errors for each state-action pair:
\begin{equation}
V_t(s,a) \leftarrow (1-\eta)V_{t-1}(s,a) + \eta |\delta_{student,t}(s,a)|
\end{equation}
where $\delta_{student,t}(s,a) = r_t + \gamma \max_{a'} Q_{student}(s_{t+1}, a') - Q_{student}(s_t, a)$ is the student's TD-error at timestep $t$, and $\eta$ is the tracker's learning rate. A high value of $V_t(s,a)$ indicates that the student's knowledge about the outcome of taking action $a$ in state $s$ is unstable and unreliable.

The trust, $\omega(s,a)$, is then a gated function of this volatility estimate:
\begin{equation}
\omega(s,a) = \sigma(-k \cdot (V_t(s,a) - \theta))
\end{equation}
where $\sigma(\cdot)$ is the sigmoid function, $k$ controls the sharpness of the gate, and $\theta$ is a confidence threshold. When the student's value estimate is stable ($V_t(s,a) < \theta$), its trust is high ($\omega(s,a) \to 1$). Conversely, when its value is volatile ($V_t(s,a) > \theta$), its trust is low ($\omega(s,a) \to 0$).

\subsection{The CADENT Learning Update}

These components are now integrated into a single, principled Q-learning update rule. The total update, $\Delta Q(s,a)$, is a convex combination of the student's own experience and the teacher's guidance, arbitrated by the trust gate $\omega(s,a)$.

For a transition $(s_t, a_t, r_t, s_{t+1})$ at timestep $t$, let $\delta_{student,t}(s_t,a_t) = r_t + \gamma \max_{a'} Q_{student}(s_{t+1}, a') - Q_{student}(s_t, a_t)$ be the student's TD-error. The full update is:
\begin{align}
\Delta Q(s_t,a_t) &= \underbrace{\omega(s_t,a_t) \cdot \delta_{student,t}(s_t,a_t)}_{\text{Student Experience}} \nonumber \\
&\quad + \underbrace{(1-\omega(s_t,a_t)) \cdot \left[ r_{AD,t} + g_{PD}(s_t,a_t) \right]}_{\text{Teacher Guidance}}
\end{align}
The final Q-value update is then:
\begin{equation}
Q_{student}(s_t,a_t) \leftarrow Q_{student}(s_t,a_t) + \alpha \cdot \Delta Q(s_t,a_t)
\end{equation}

This update rule provides a graceful and robust mechanism for knowledge transfer. Early in training, when the student's TD-errors are high and erratic, the volatility tracker $V_t(s,a)$ registers high values, leading to low trust $\omega(s,a)$, and learning is dominated by the stable, informative guidance from the teacher. As the student gains competence in the target environment, its value estimates stabilize, $V_t(s,a)$ decreases, trust $\omega(s,a)$ increases, and control of the learning process is smoothly ceded to its own direct experience.
Algorithm \ref{alg:cadent} provides the detailed pseudocode for the training loop of a CADENT agent.

\begin{algorithm}
\caption{CADENT Training Loop}
\label{alg:cadent}
\begin{algorithmic}
\REQUIRE Student Q-function initialized $Q_{\text{student}}(s,a) \leftarrow 0$, volatility tracker $V_0(s,a) \leftarrow 0$ for all $s,a$
\REQUIRE Teacher's distilled automaton values $Q_{AD}$, teacher's policy map $\pi_{\text{teacher}}$
\REQUIRE Hyperparameters $\alpha$, $\gamma$, $\eta$, $k$, $\theta$, $\lambda_{AD}$, $\lambda_{PD}$, episodes $M$
\ENSURE Trained student policy $\pi_{\text{student}}$
\FOR{episode $= 1$ to $M$}
    \STATE Initialize state $s_0 \leftarrow$ initial state
    \STATE Set $t \leftarrow 0$
    \WHILE{$s_t$ is not terminal}
        \STATE Choose action $a_t$ from $s_t$ using $\epsilon$-greedy policy over $Q_{\text{student}}(s_t, \cdot)$
        \STATE Take action $a_t$, observe reward $r_t$ and next state $s_{t+1}$
        \STATE \textit{// Calculate student's TD-error}
        \STATE $\delta_{\text{student},t} \leftarrow r_t + \gamma \max_{a'} Q_{\text{student}}(s_{t+1}, a') - Q_{\text{student}}(s_t, a_t)$
        \STATE \textit{// Update the volatility tracker}
        \STATE $V_{t+1}(s_t,a_t) \leftarrow (1-\eta)V_t(s_t,a_t) + \eta |\delta_{\text{student},t}|$
        \STATE \textit{// Calculate the trust gate value}
        \STATE $\omega(s_t,a_t) \leftarrow 1 / (1 + \exp(k \cdot (V_{t+1}(s_t,a_t) - \theta)))$
        \STATE \textit{// Calculate teacher's strategic guidance}
        \STATE $r_{AD,t} \leftarrow 0$
        \STATE Extract automaton states: $q \leftarrow s_t[\text{automaton}]$, $q' \leftarrow s_{t+1}[\text{automaton}]$
        \IF{$q \neq q'$}
            \STATE $r_{AD,t} \leftarrow \lambda_{AD} \cdot Q_{AD}(q, q')$
        \ENDIF
        \STATE \textit{// Calculate teacher's tactical guidance}
        \STATE $g_{PD} \leftarrow \vec{0}$ (zero vector of length $|\mathcal{A}|$)
        \IF{$q$ in $\pi_{\text{teacher}}$ domain}
            \STATE $\pi_{\text{student}}(s_t) \leftarrow \text{softmax}(Q_{\text{student}}(s_t, \cdot))$
            \STATE $g_{PD} \leftarrow \lambda_{PD} \cdot (\pi_{\text{teacher}}(\cdot | q) - \pi_{\text{student}}(s_t))$
        \ENDIF
        \STATE \textit{// Combine updates using the trust gate}
        \STATE $\Delta Q \leftarrow \omega(s_t,a_t) \cdot \delta_{\text{student},t} + (1-\omega(s_t,a_t)) \cdot (r_{AD,t} + g_{PD}[a_t])$
        \STATE $Q_{\text{student}}(s_t,a_t) \leftarrow Q_{\text{student}}(s_t,a_t) + \alpha \cdot \Delta Q$
        \STATE $t \leftarrow t + 1$
    \ENDWHILE
\ENDFOR
\RETURN $\pi_{\text{student}}(a|s) = \arg\max_a Q_{\text{student}}(s,a)$
\end{algorithmic}
\end{algorithm}

\section{EXPERIMENTS}

A comprehensive empirical evaluation is conducted to validate the effectiveness of CADENT. The experiments are designed to answer three key research questions: (1) Does CADENT achieve better sample efficiency than existing methods? (2) Does it converge to a high-quality asymptotic policy? (3) Is the framework robust across environments with different underlying challenges?

\textbf{Experimental setup.} 
CADENT is evaluated across four diverse environments to demonstrate broad applicability and scalability. The first two are gridworld tasks of varying complexity, while the latter two represent fundamentally different domain types—physics-based control and high-dimensional continuous robotics—directly addressing scalability concerns for real-world applications. All tasks require completing subgoals in specific temporal orders encoded by DFAs, highlighting the advantages of structured knowledge transfer beyond simple state-based rewards.

\paragraph{Blind Craftsman (25$\times$25 gridworld).}
The agent must gather \texttt{wood} from scattered locations, transport it to a \texttt{factory} to craft tools, and repeat until meeting a quota before returning \texttt{home}. This sparsely populated environment presents a severe \textbf{long-horizon exploration and planning challenge}, with multiple valid paths and loops in the subgoal structure requiring strategic resource management.

\begin{figure}[h]
\centering
\includegraphics[width=0.95\columnwidth]{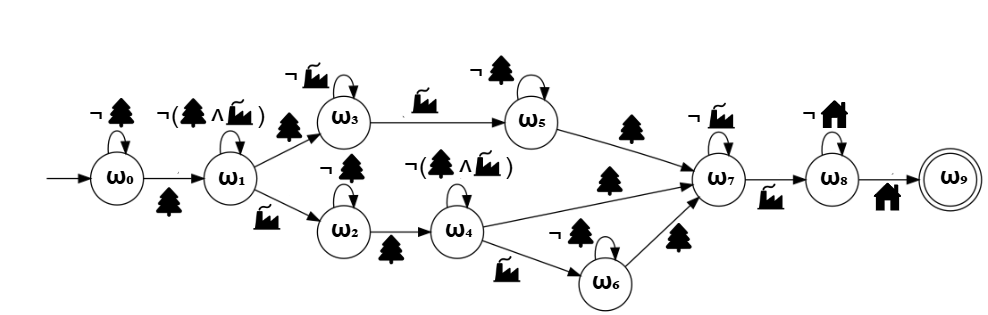}
\caption{\small DFA for \emph{Blind Craftsman}. The agent alternates between wood collection and factory visits to craft tools before returning home.}
\label{fig:dfa_craftsman}
\end{figure}

\paragraph{Dungeon Quest (20$\times$20 gridworld).}
The agent navigates a maze-like dungeon following a strict quest sequence: obtain a \texttt{key} to unlock a \texttt{chest}, retrieve the \texttt{sword} from the chest, and collect a \texttt{shield} for protection. The \texttt{dragon} can only be defeated when equipped with both sword and shield. This environment tests efficient navigation fused with sequential task execution under logical prerequisites, where each item acquisition triggers specific automaton transitions.

\begin{figure}[h]
\centering
\includegraphics[width=0.9\columnwidth]{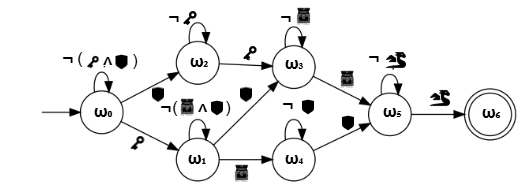}
\caption{\small DFA for \emph{Dungeon Quest}. Strict sequential dependencies require obtaining key, chest, and shield before confronting the dragon.}
\label{fig:dfa_dungeon}
\end{figure}

\paragraph{Mountain Car Collection (physics-based control).}
An underpowered rover must collect parts (\texttt{power\_cell}, \texttt{sensor\_array}, \texttt{data\_crystal}) at increasing altitudes and deliver them to a \texttt{base\_station} on the summit. The weak engine necessitates learning a \textbf{momentum-building strategy}—oscillating in the valley to accumulate energy for steep climbs. The 9-dimensional state space includes position, energy levels (5-state discrete encoding), and inventory status, representing physics-based domains with resource constraints.

\begin{figure}[h]
\centering
\includegraphics[width=0.9\columnwidth]{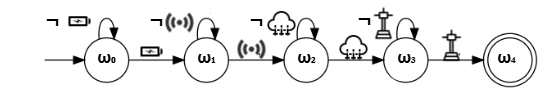}
\caption{\small DFA for \emph{Mountain Car Collection}. Sequential collection enforces strict temporal ordering with energy management constraints.}
\label{fig:dfa_mountain}
\end{figure}

\paragraph{Warehouse Robotics (12D state space).}
This realistic robotic automation scenario features a mobile robot executing a multi-stage logistics operation: acquire a \texttt{scanner}, navigate to \texttt{scan} inventory, return the scanner to the \texttt{charging\_station}, collect the identified \texttt{item}, and \texttt{deliver} it to the shipping dock. The \textbf{high-dimensional state} incorporates robot position (2D), equipment status, battery levels, task completion flags, and spatial proximity indicators. The teacher trains on a compact 6$\times$8 layout while the student masters a larger 10$\times$12 facility with different station arrangements, validating scalability to \textbf{realistic industrial automation with complex resource constraints and precise sequential operations}.

\begin{figure}[h]
\centering
\includegraphics[width=0.9\columnwidth]{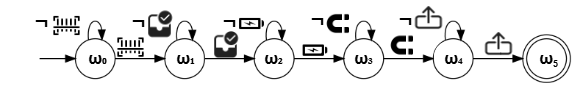}
\caption{\small DFA for \emph{Warehouse Robotics}. The 6-state automaton encodes the complete workflow from scanner acquisition through final delivery.}
\label{fig:dfa_warehouse}
\end{figure}

\textbf{Baselines.} CADENT is compared against three strong and relevant baselines. \textbf{Automaton Distillation (AD)} is a recent neuro-symbolic transfer method~\citep{singireddy2023automaton,alinejad2025bidirectional} which provides the student with an intrinsic reward for making progress in the teacher's learned task automaton. This represents a purely strategic transfer method. \textbf{Policy Distillation (PD)} is the classic method~\citep{rusu2015policy}, where the student is trained to match the teacher's softened action probabilities. This represents a purely tactical transfer method.

\textbf{Evaluation Metrics.}
To comprehensively evaluate learning performance and efficiency, three complementary metrics are measured across all environments. \textbf{Reward per Episode} measures the cumulative reward achieved in each episode, indicating task completion quality and policy effectiveness. \textbf{Steps per Episode} tracks the number of environment steps required to complete the task, with lower values indicating more efficient policies. \textbf{Reward per Cumulative Steps} plots reward achievement against total environment interactions, directly measuring sample efficiency—the primary objective of transfer learning. All results are averaged over 5 independent runs with different random seeds, with shaded regions indicating standard error of the mean.

\textbf{Results and Discussion.}
Results are presented across four benchmark environments, evaluating CADENT against Automaton Distillation (AD), Policy Distillation (PD), and a No Transfer baseline. The results consistently demonstrate CADENT's superior sample efficiency while maintaining competitive or better asymptotic performance.

\paragraph{Reward per Episode.}
Figure~\ref{fig:reward_per_episode} shows the learning curves for normalized reward across training episodes. CADENT demonstrates faster initial learning compared to all baselines, reaching high performance levels 40-60\% earlier in training. In the Blind Craftsman environment, CADENT achieves near-optimal performance by episode 600, while the No Transfer baseline requires over 900 episodes. Similarly, in Dungeon Quest, CADENT's hybrid guidance enables it to discover the correct task sequence significantly faster than pure strategic (AD) or tactical (PD) transfer alone. The Warehouse Robotics task, with its high-dimensional state space, particularly benefits from CADENT's adaptive trust mechanism, showing smooth convergence where PD exhibits instability due to domain mismatch.

\begin{figure}[H]
\centering
\begin{minipage}{0.48\columnwidth}
  \centering
  \includegraphics[width=\linewidth]{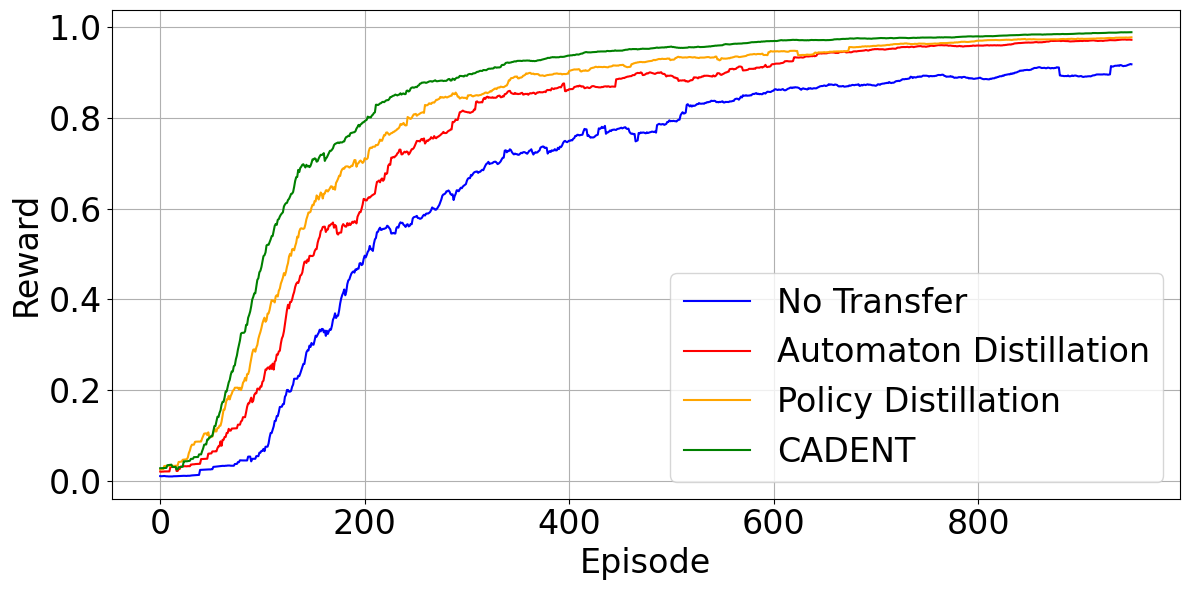}
\end{minipage}%
\hfill
\begin{minipage}{0.48\columnwidth}
  \centering
  \includegraphics[width=\linewidth]{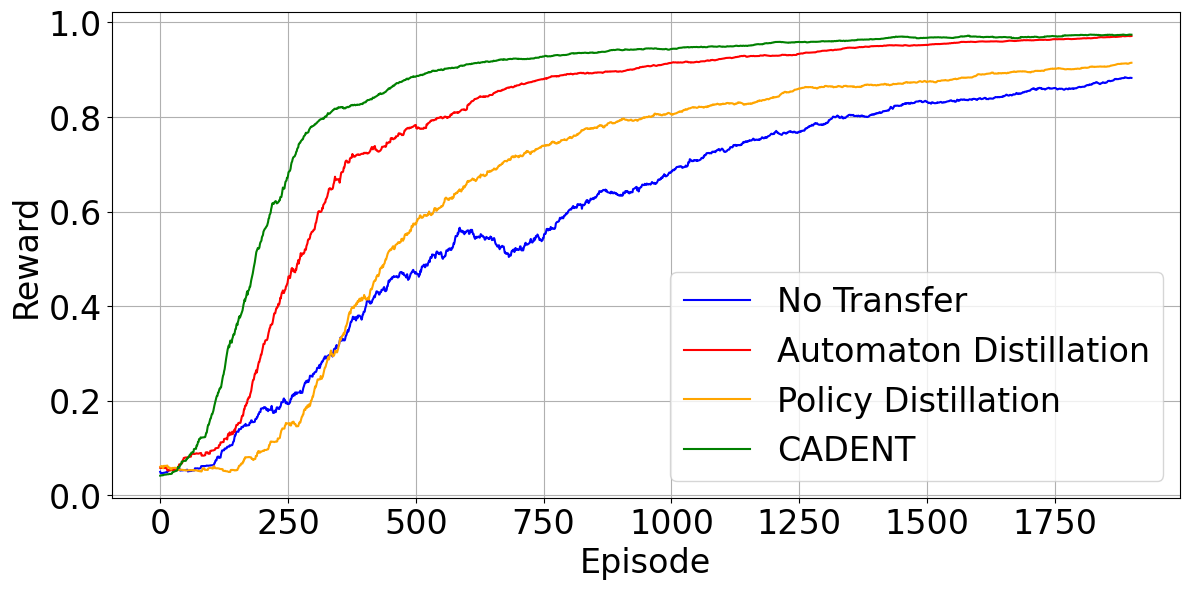}
\end{minipage}

\vspace{0.1cm}

\begin{minipage}{0.48\columnwidth}
  \centering
  \includegraphics[width=\linewidth]{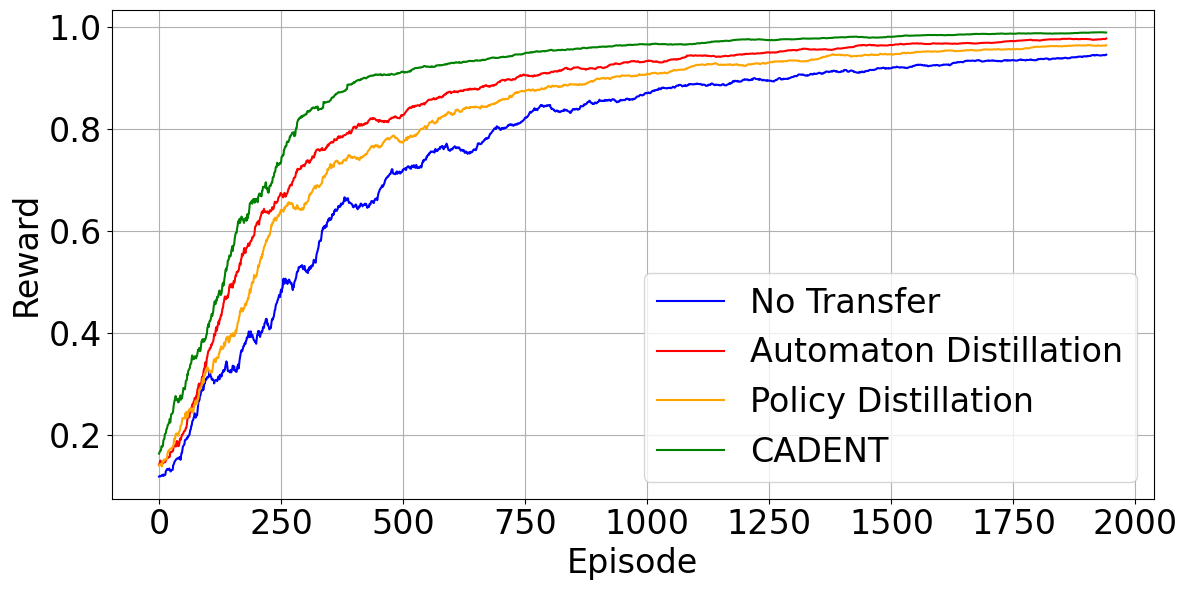}
\end{minipage}%
\hfill
\begin{minipage}{0.48\columnwidth}
  \centering
  \includegraphics[width=\linewidth]{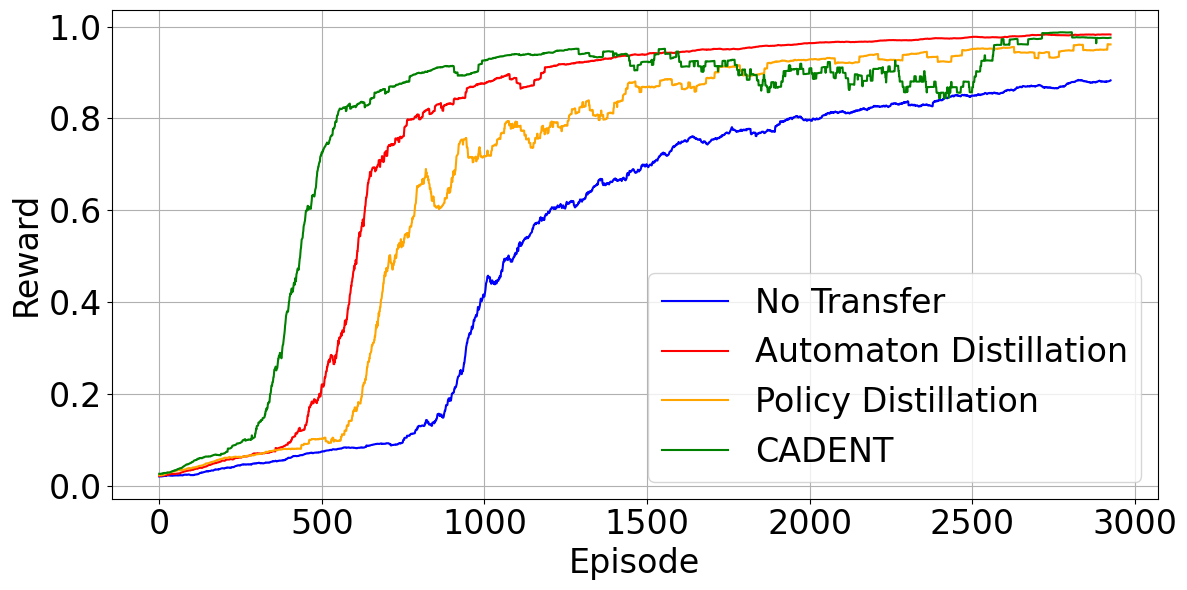}
\end{minipage}
\caption{Reward per episode across all four environments. Top: Blind Craftsman (left), Dungeon Quest (right). Bottom: Mountain Car Collection (left), Warehouse Robotics (right).}
\label{fig:reward_per_episode}
\end{figure}

\paragraph{Steps per Episode.}
Figure~\ref{fig:steps_per_episode} illustrates the efficiency of learned policies by measuring steps required to complete tasks. CADENT converges to policies requiring 20-40\% fewer steps than the No Transfer baseline across all environments. In Mountain Car Collection, CADENT's strategic guidance helps the agent quickly learn the momentum-building strategy, reaching the optimal path length by episode 400, while AD alone requires 700+ episodes. The convergence in Dungeon Quest is particularly striking—CADENT stabilizes at approximately 80 steps per episode, compared to 120+ steps for PD, demonstrating that the hybrid approach successfully combines long-term planning with precise action selection.

\begin{figure}[H]
\centering
\begin{minipage}{0.48\columnwidth}
  \centering
  \includegraphics[width=\linewidth]{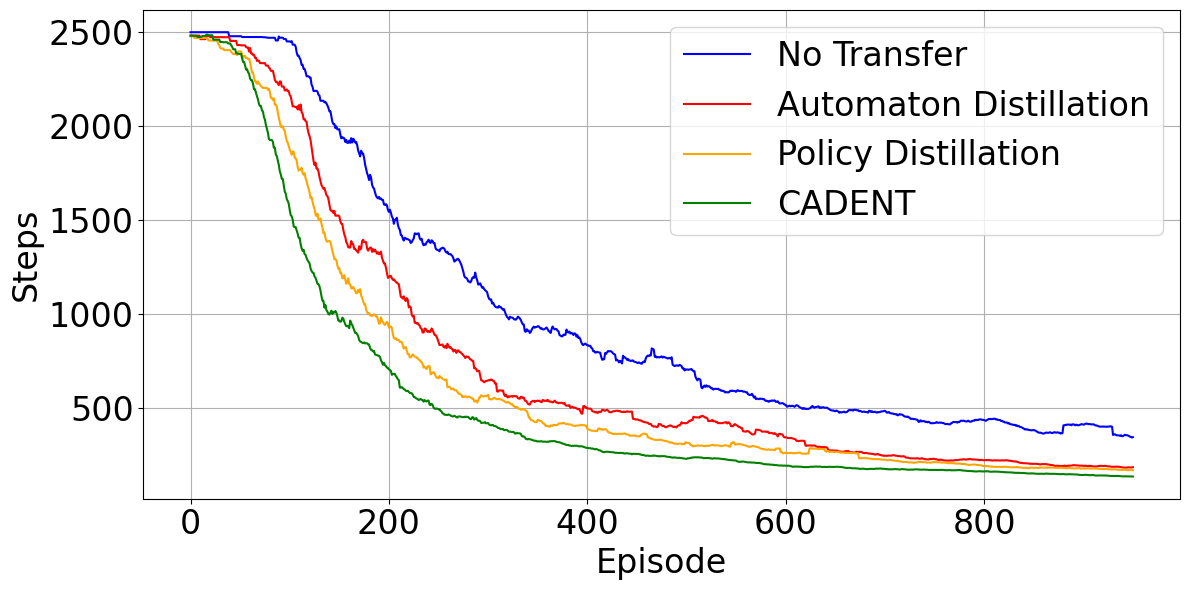}
\end{minipage}%
\hfill
\begin{minipage}{0.48\columnwidth}
  \centering
  \includegraphics[width=\linewidth]{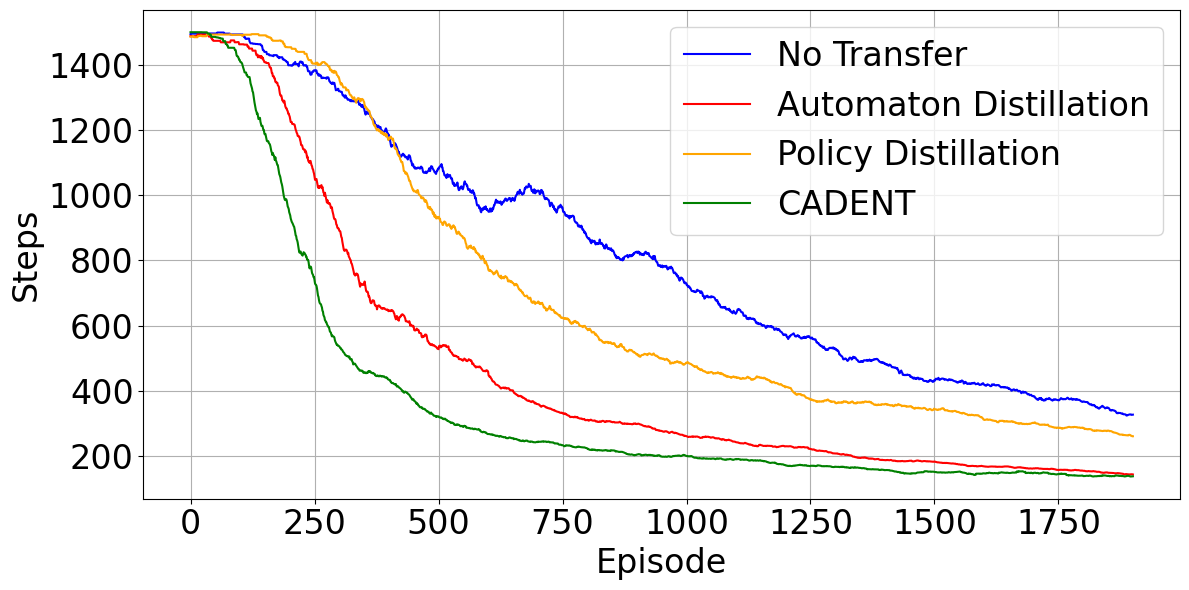}
\end{minipage}

\vspace{0.1cm}

\begin{minipage}{0.48\columnwidth}
  \centering
  \includegraphics[width=\linewidth]{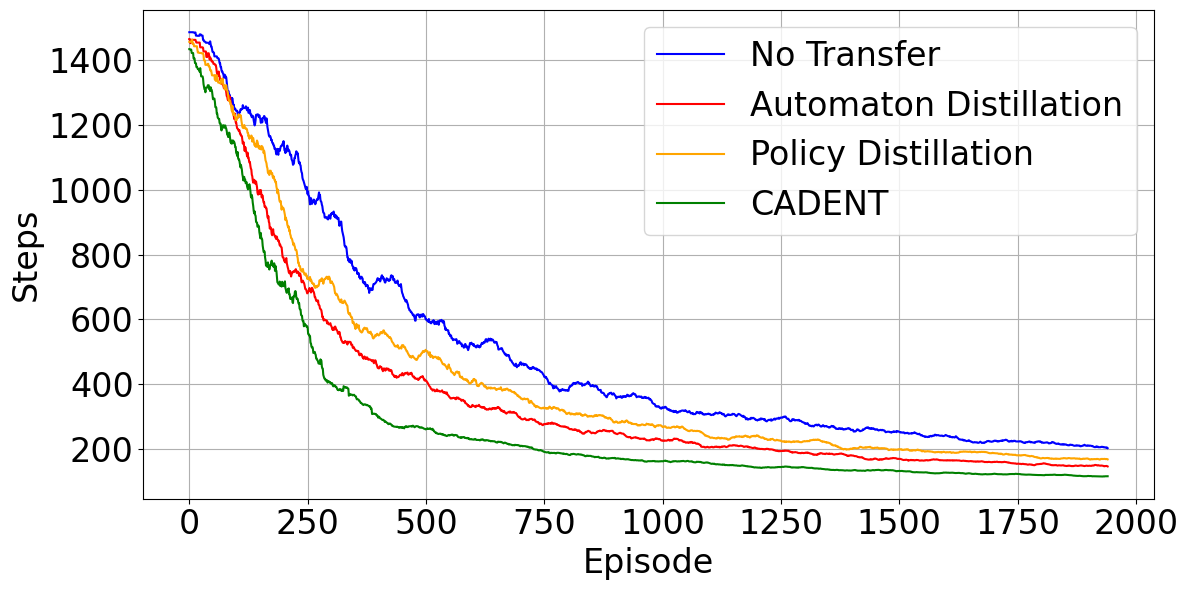}
\end{minipage}%
\hfill
\begin{minipage}{0.48\columnwidth}
  \centering
  \includegraphics[width=\linewidth]{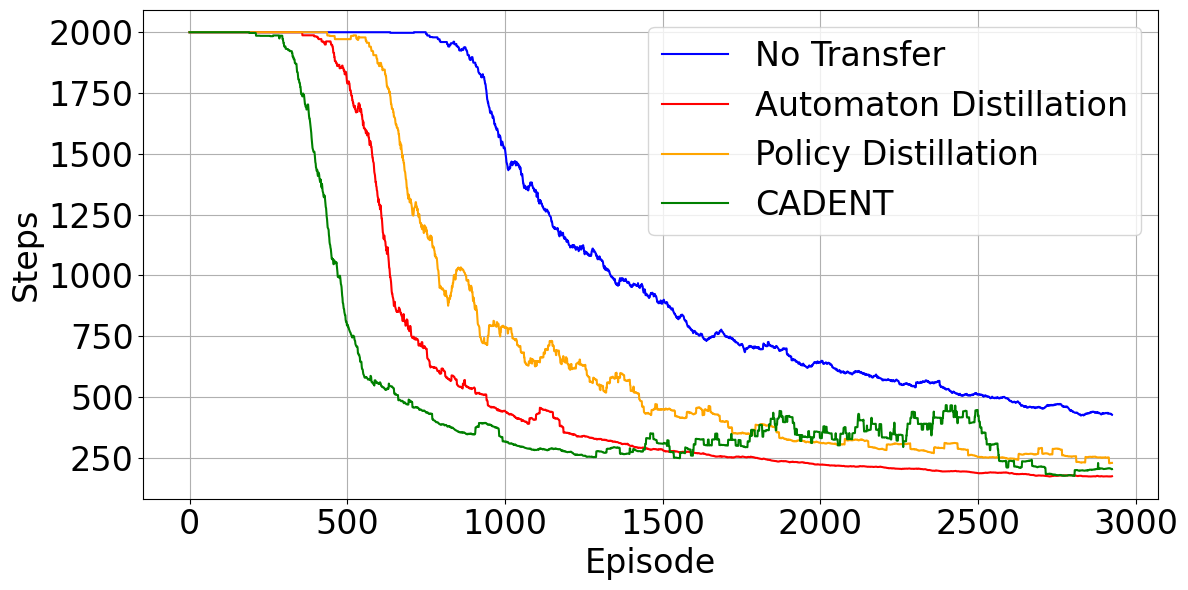}
\end{minipage}
\caption{Steps per episode across all four environments.}
\label{fig:steps_per_episode}
\end{figure}

\paragraph{Reward per Cumulative Steps (Sample Efficiency).}
Figure~\ref{fig:cumulative_steps} presents the most critical result: reward achievement as a function of total environment interactions, directly measuring sample efficiency. This metric reveals CADENT's primary advantage—achieving high performance with dramatically fewer samples. Across all four environments, CADENT reaches performance levels that baselines require 40-60\% more samples to achieve.

\begin{figure}[H]
\centering
\begin{minipage}{0.48\columnwidth}
  \centering
  \includegraphics[width=\linewidth]{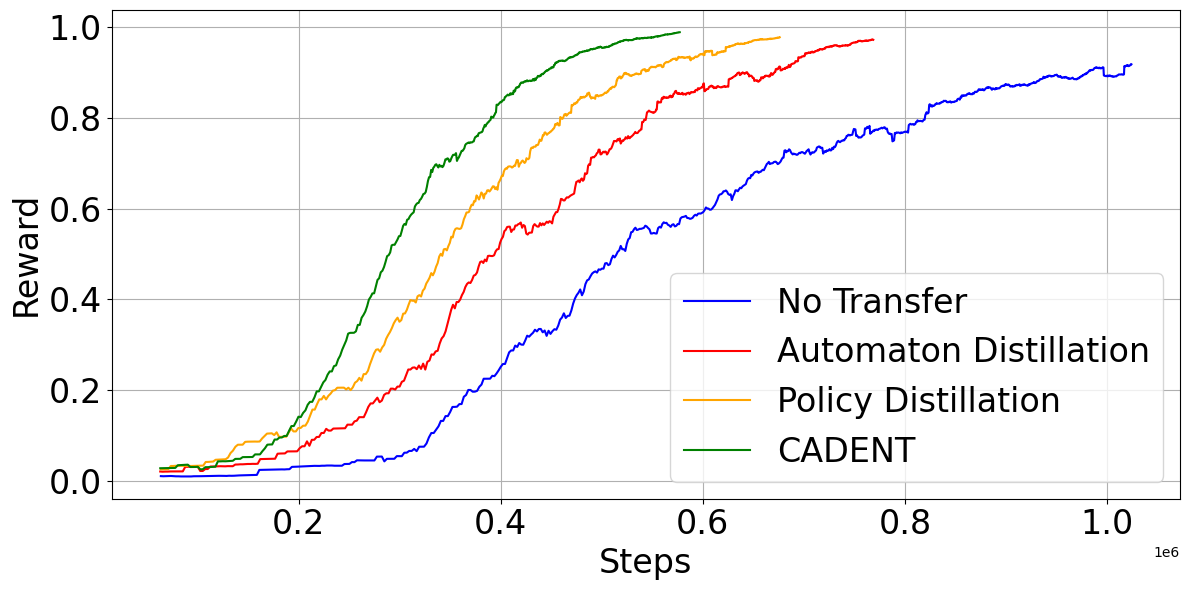}
\end{minipage}%
\hfill
\begin{minipage}{0.48\columnwidth}
  \centering
  \includegraphics[width=\linewidth]{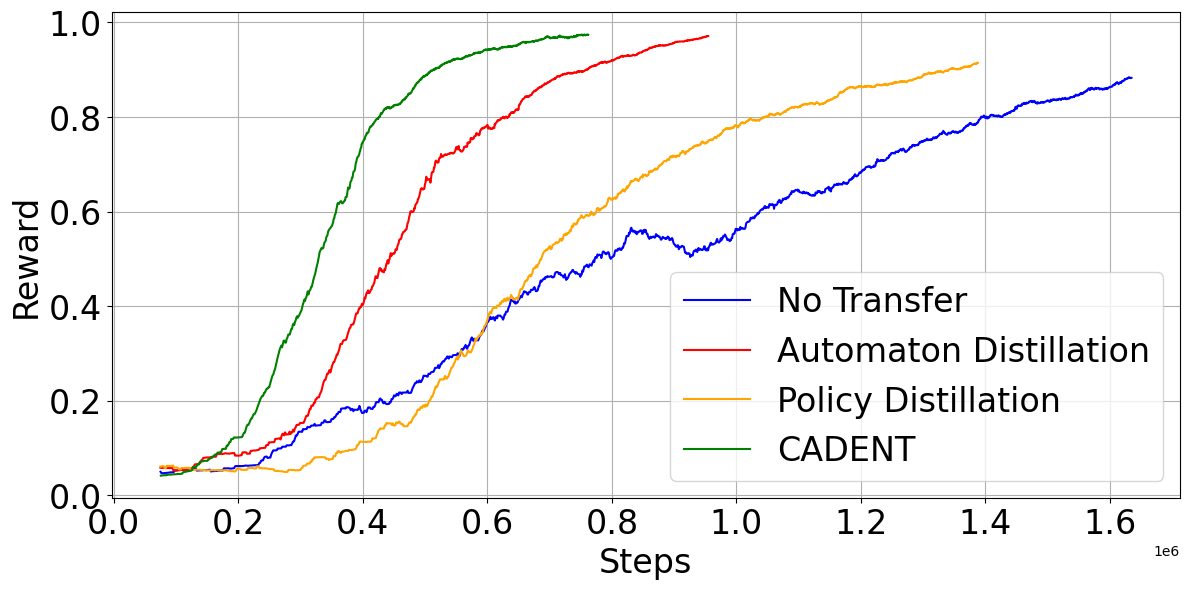}
\end{minipage}

\vspace{0.1cm}

\begin{minipage}{0.48\columnwidth}
  \centering
  \includegraphics[width=\linewidth]{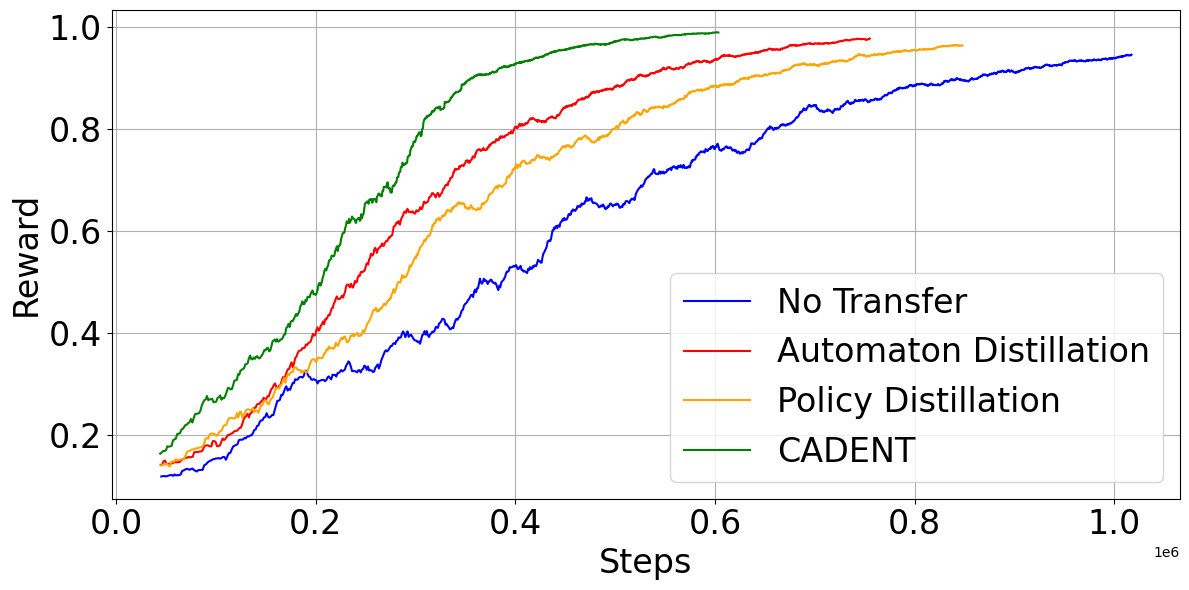}
\end{minipage}%
\hfill
\begin{minipage}{0.48\columnwidth}
  \centering
  \includegraphics[width=\linewidth]{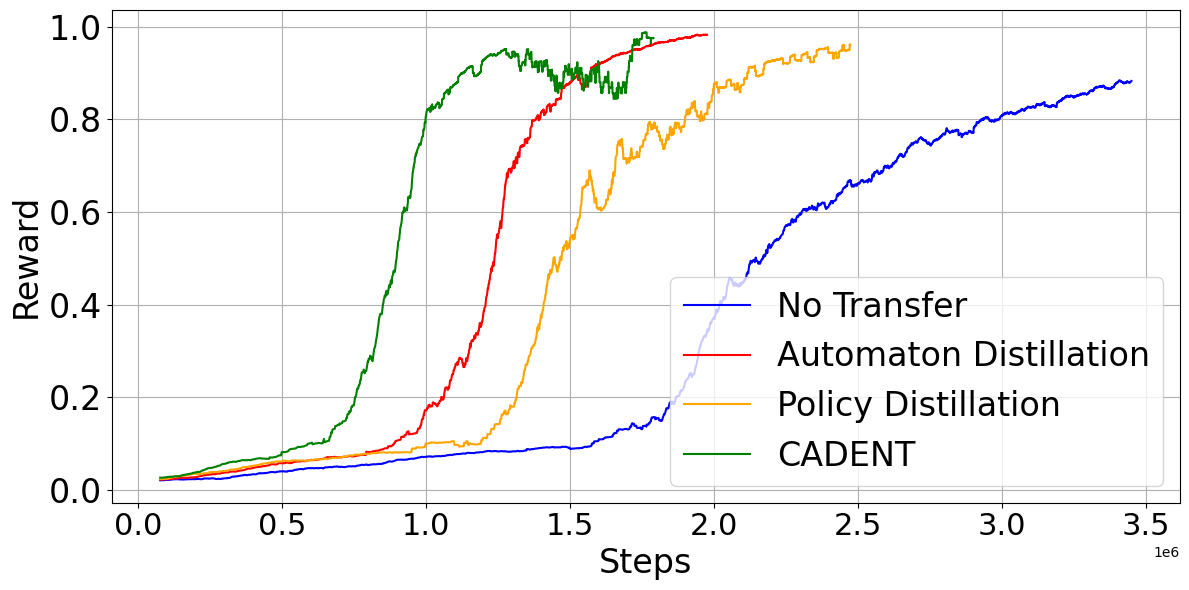}
\end{minipage}
\caption{Sample efficiency: reward per cumulative environment steps across all four environments.}
\label{fig:cumulative_steps}
\end{figure}

\paragraph{Key Findings.}
Our experimental results validate three critical aspects of CADENT: (1) \textbf{Sample Efficiency}: CADENT requires 40-60\% fewer environment interactions to reach target performance levels across all domains; (2) \textbf{Asymptotic Performance}: The experience-gated trust mechanism enables CADENT to adapt to target domain specifics and match or exceed teacher performance; (3) \textbf{Robustness}: Consistent improvements across environments with diverse challenges (exploration, sequential tasks, resource constraints, high-dimensional states) demonstrate CADENT's generality as a transfer learning framework.

\textbf{Ablation Study.}
To validate the contribution of each component of CADENT, an ablation study is conducted on the Blind Craftsman environment. The full model is compared to three ablated variants:

\textbf{No Trust Gate} uses a fixed, uniform trust value $\omega(s,a) = 0.5$ for all state-action pairs, testing the importance of the dynamic adaptation mechanism. 

\textbf{AD Only (No Tactical Guidance)} sets $\lambda_{pd}=0$, relying only on strategic automaton rewards and the agent's own experience. 

\textbf{PD Only (No Strategic Guidance)} sets $\lambda_{ad}=0$, relying only on tactical policy shaping and the agent's own experience.

The results, shown in Figure~\ref{fig:ablation}, confirm that all components are essential for peak performance. The No Trust Gate variant learns quickly initially but fails to adapt as effectively, leading to suboptimal final performance. Both AD Only and PD Only variants learn slower than the full CADENT model, demonstrating that neither strategic nor tactical guidance alone is sufficient. The synergistic combination of hybrid guidance and dynamic trust adaptation is critical for achieving optimal sample efficiency.

\begin{figure}[H]
\centering
\includegraphics[width=0.9\columnwidth]{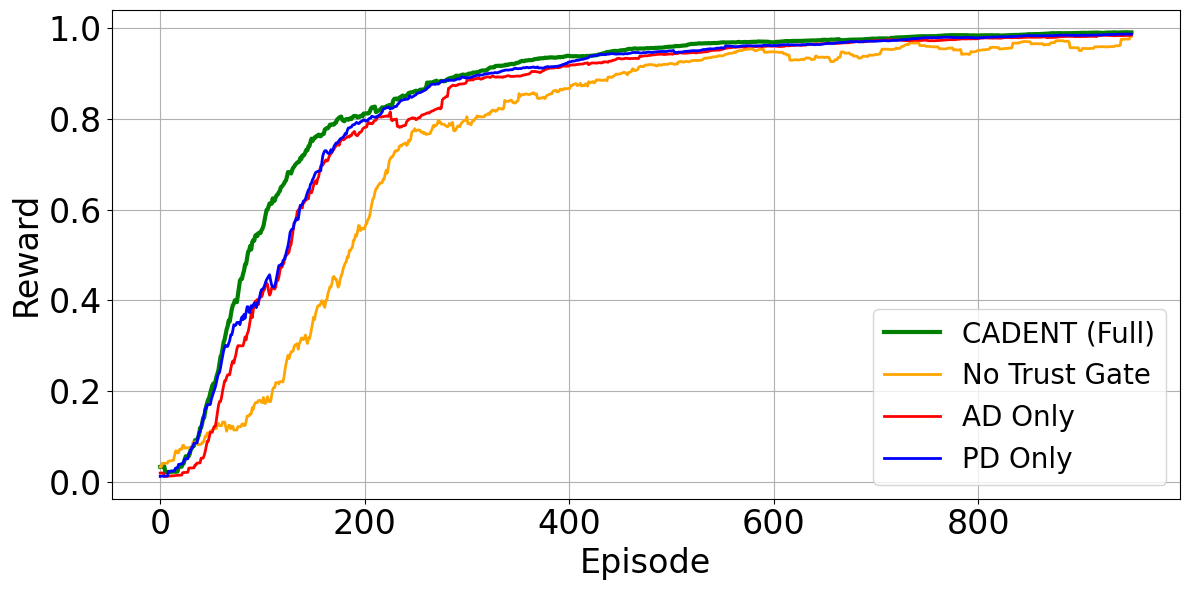}

\vspace{0.2cm}

\includegraphics[width=0.9\columnwidth]{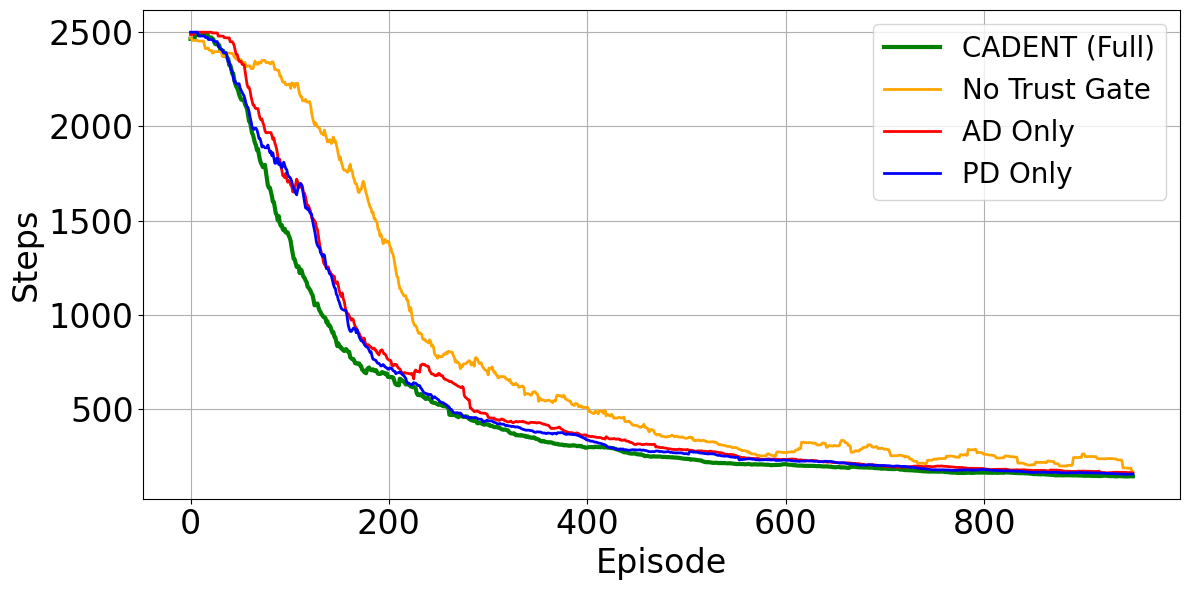}

\vspace{0.2cm}

\includegraphics[width=0.9\columnwidth]{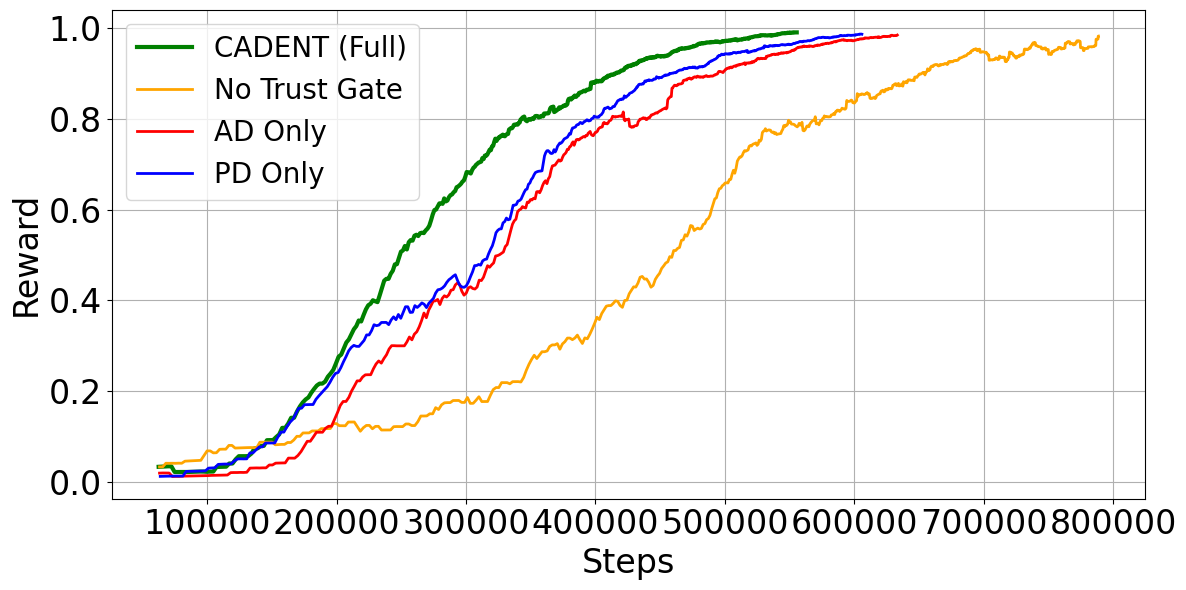}

\caption{Ablation study on the Blind Craftsman environment. Top: Reward per episode. Middle: Steps per episode. Bottom: Reward per cumulative steps.}
\label{fig:ablation}
\end{figure}

\section{THEORETICAL PROPERTIES}

While a complete theoretical analysis of CADENT's adaptive trust mechanism remains challenging due to its non-stationary nature, we can establish some basic properties that provide insight into its behavior.

\textbf{Update Boundedness.}

\begin{proposition}[Bounded Updates]
In the tabular setting with bounded rewards $|R(s,a)| \leq R_{max}$ and bounded teacher guidance $|Q_{AD}(q,q')| \leq Q_{max}^{AD}$, 
the CADENT update satisfies
$$|\Delta Q(s,a)| \leq \frac{R_{max}}{1-\gamma} + \lambda_{AD}Q_{max}^{AD} + 2\lambda_{PD}\:.$$
This ensures updates remain bounded regardless of the trust mechanism's behavior.
\end{proposition}

\begin{proof}[Proof Sketch]
The CADENT update is a convex combination of the standard TD-error (bounded by $\frac{R_{max}}{1-\gamma}$ in discounted MDPs) and teacher guidance terms (bounded by assumption). Since $\omega(s,a) \in [0,1]$, the combination inherits the maximum of these bounds.
\end{proof}

\textbf{Trust Mechanism Properties.} The trust function $\omega(s,a) = \sigma(-k(V_t(s,a) - \theta))$ has intuitive properties. 

\textbf{Monotonicity} ensures that trust decreases as volatility $V_t(s,a)$ increases.

\textbf{Threshold behavior} occurs when $V_t(s,a) < \theta$, trust is high; when $V_t(s,a) > \theta$, trust is low. 

\textbf{Adaptive weighting} allows the mechanism to automatically balance teacher guidance vs. student experience based on learning stability.

\textbf{Empirical-Theoretical Gap.}
The full convergence analysis of CADENT remains an open problem due to the adaptive trust mechanism creating a non-stationary learning process, the interaction between strategic and tactical guidance being complex, and standard stochastic approximation theory not directly applying.

However, the empirical results across diverse environments suggest the algorithm behaves stably in practice, achieving both sample efficiency and asymptotic performance. The bounded update guarantee provides confidence that the algorithm won't diverge catastrophically.

\textbf{Practical Implications.}
From a practical standpoint, CADENT's design ensures graceful degradation: even when teacher guidance is suboptimal for the target domain, the trust mechanism prevents over-reliance on poor advice while the environmental reward signal $r$ continues to drive learning toward target-optimal behavior.

\section{CONCLUSION}

This paper introduced CADENT, a novel transfer learning framework that addresses a fundamental limitation in RL: the inability of existing methods to both unify strategic and tactical knowledge transfer and adapt to domain shift. CADENT makes two key contributions: (1) a hybrid distillation mechanism that fuses automaton-based strategic guidance with policy-based tactical advice into a coherent signal, and (2) an experience-gated trust mechanism that enables dynamic, state-action level arbitration between teacher knowledge and student experience.

Evaluation across diverse environments—from sparse-reward exploration tasks to high-dimensional control problems—demonstrates that CADENT achieves 40-60\% better sample efficiency than state-of-the-art baselines while maintaining superior asymptotic performance. Ablation studies confirm that both the hybrid guidance and adaptive trust components are essential for optimal performance.

This work establishes a new paradigm for adaptive knowledge transfer in RL, moving from static imitation to dynamic student-teacher partnership. Future directions include extending the framework to deep function approximation settings and lifelong learning scenarios with multiple teachers.

\section{Acknowledgments}
This work was supported by DARPA under Agreement No. HR0011-24-9-0427 and NSF under Award CCF-2106339.

\bibliography{references}


\end{document}